\documentclass{article}

    \usepackage[final]{neurips_2021}

\usepackage[utf8]{inputenc} 
\usepackage[T1]{fontenc}    
\usepackage{hyperref}       
\usepackage{url}            
\usepackage{booktabs}       
\usepackage{amsfonts}       
\usepackage{nicefrac}       
\usepackage{microtype}      
\usepackage{xcolor}         

\usepackage{algorithm, algorithmic}
\usepackage{natbib}
\usepackage{graphicx}
\usepackage{subfigure}
\usepackage{multirow}
\usepackage{amsmath, bm, amsthm}
\usepackage{wrapfig}
\usepackage{xfrac}
\usepackage{sidecap}
\usepackage{adjustbox}
\usepackage{caption}
\usepackage{bbm}

\newcommand{\KL}{{\mathrm{KL}}}

\newcommand{\mean}{\mathbb{E}}
\newcommand{\var}{{\rm I\kern-.3em D}}

\newcommand{\Normal}{\mathcal{N}}

\newcommand{\norm}[1]{\left\lVert#1\right\rVert}
\newcommand{\eps}{\varepsilon}

\DeclareMathOperator*{\maxargmax}{max\cdot arg\,max}

\newtheorem*{theorem*}{Theorem}

\newtheorem{proposition}{Proposition}
\newtheorem*{proposition*}{Proposition}

\DeclareMathSymbol{\shortminus}{\mathbin}{AMSa}{"39}

\usepackage{color}

\title{Particle Dynamics for Learning EBMs}

\author{%
  Kirill Neklyudov
    \\
  University of Amsterdam, \\
  \texttt{k.necludov@gmail.com} \\
   \And
   Priyank Jaini\thanks{the work was done while at University of Amsterdam} \\
   Google Brain\\
   \texttt{pjaini@google.com} \\
   \And
   Max Welling \\
   University of Amsterdam\\
}

\begin{document}

\maketitle

\begin{abstract}
  Energy-based modeling is a promising approach to unsupervised learning, which yields many downstream applications from a single model. The main difficulty in learning energy-based models with the ``contrastive approaches'' is the generation of samples from the current energy function at each iteration. Many advances have been made to accomplish this subroutine cheaply. Nevertheless, all such sampling paradigms run MCMC targeting the current model, which requires infinitely long chains to generate samples from the true energy distribution and is problematic in practice. This paper proposes an alternative approach to getting these samples and avoiding crude MCMC sampling from the current model. We accomplish this by viewing the evolution of the modeling distribution as (i) the evolution of the energy function, and (ii) the evolution of the samples from this distribution along some vector field. We subsequently derive this time-dependent vector field such that the particles following this field are approximately distributed as the current density model. Thereby we match the evolution of the particles with the evolution of the energy function prescribed by the learning procedure. Importantly, unlike Monte Carlo sampling, our method targets to match the current distribution in a finite time. Finally, we demonstrate its effectiveness empirically comparing to MCMC-based learning methods.
\end{abstract}

\section{Introduction}

Energy-based modeling has recommended itself as a universal approach learning a single model, which then can be applied in various scenarios: continual learning, missing data imputation, out-of-distribution detection, better uncertainty of discriminative models \citep{grathwohl2019your, du2019implicit, li2020energy}.
However, scaling this approach to real-world data such as images encounters many complications, which the community has been approaching by trying to get better samples
\citep{tieleman2009using, du2019implicit, nijkamp2019learning} or by targeting different objectives
\citep{grathwohl2020learning, arbel2020generalized, gao2020learning}.

In this paper, we approach the subroutine problem of getting samples from the current model, which arises in the learning of the energy-based models.
The conventional approach to this is to run an MCMC method targeting the current model.
Instead, we update particles deterministically propagating them along the derived vector field such that after time $dt$ the particles are distributed as the evolved density after time $dt$.
This is principally different, since we don't rely on the convergence to the target (the current model density), and are able to match it in a finite amount of time.
Our main contribution is the formula \eqref{eq:main_eq} for the vector field, which matches the evolution of the model with the evolution of the particles in the space of log-densities.
Further, we discuss possible ways to simulate this formula and demonstrate its usefulness empirically.

\section{Background and Related Works}

\textbf{Energy-Based Models} are usually learned via the maximum likelihood principle.
That is, we start with a model density function $q(x)$ parameterized by the energy function $E(x,\theta)$:
\begin{align}
    q_{\theta}(x) = \frac{1}{Z} e^{-E(x,\theta)}, \;\;\; Z = \int dx\; e^{-E(x,\theta)},
\end{align}
which is then optimized to approximate some target density $p(x)$ given empirically (as a set of samples).
This can be done by the maximization of $\mean_p \log q$, or, equivalently, minimization of $\KL(p,q)$ via the gradient methods:
\begin{align}
    -\nabla_\theta \KL(p,q_\theta) = -\nabla_\theta\bigg[ \mean_{x\sim p} E(x,\theta) - \mean_{x\sim q_\theta} E(x,\theta)\bigg],
\end{align}
The main obstacle under this approach is the sampling from the current density $q_\theta \propto \exp(-E(x,\theta))$.
Our work operates much in the fashion of the Persistent Contrastive Divergence (PCD) \citep{tieleman2009using}. 
It keeps a set of samples, which are updated at every iterations to match current $q_\theta$.
While PCD relies on MCMC methods targeting $q_\theta$, we propagate the particles deterministically along the derived vector field.

\textbf{Langevin Dynamics} is a ubiquitous sampling method.
For energy-based models with the continuous state-space, this method is especially attractive due to its cheap iterations (single gradient evaluation per step) and the ability to yield good samples even without Metropolis-Hastings correction \citep{gelfand1991recursive}.
This procedure targeting the density $p$ can be written as 
\begin{align}
    x_{t+dt} = x_t + dt\frac{1}{2}\nabla_x\log p(x) + \eps, \;\;\; \eps \sim \Normal(0,dt)
\end{align}
Its efficiency, however, is hindered by the random fluctuations that introduce random-walk behaviour, and its deterministic analog is more efficient (see, for instance, \citep{liu2019understanding}).
This analog is derived by rewriting the Fokker-Planck equation (which describes the evolution of the density) as the continuity equation:
\begin{align}
    \frac{\partial q}{\partial t} = -\langle \nabla, q\frac{1}{2}\nabla\log p \rangle + \frac{1}{2}\Delta q = -\bigg\langle \nabla, q(\frac{1}{2}\nabla\log p-\frac{1}{2}\nabla \log q) \bigg\rangle.
\end{align}
Then the simulation of the particles can be done by propagating them along the new vector field:
\begin{align}
    x_{t+dt} = x_t + \frac{dt}{2}\big[\nabla \log p(x_t)-\nabla \log q_t(x_t)\big].
    \label{eq:det_langevin}
\end{align}
In Monte Carlo setting, the deterministic simulation is troublesome since we don't have an access to the current density $q_t$. However, we will see how the EBMs learning naturally allows for this.


\section{Matching the particle dynamics with the energy evolution}
\label{sec:motivation}

In this section, we try to match two things: the update of the energy and the update of the particles. 
The former is defined by the learning procedure maximizing the log-likelihood.
The particles then should be propagated to keep up with the updates of energy and be distributed as the most recent model.
We match these two dynamics by matching the updates of their log-densities in $L_q^2$:
\begin{align}
    v^* = \maxargmax_{v\in L^2_q: \norm{v}=1} \bigg\langle \frac{\partial}{\partial t} \log q_t, \frac{\partial}{\partial t} \log \hat{q}_t \bigg\rangle_{L_q^2},
    \label{eq:projection}
\end{align}
where ``$\maxargmax$'' denotes the scalar multiplication of
the maximum and the maximizer, $q_t$ is the prescribed evolution, and $\hat{q}_t$ is the density evolution of particles defined by the vector field $v$, i.e.
\begin{align}
    \frac{\partial}{\partial t}\log \hat{q}_t = \frac{1}{\hat{q}_t}\frac{\partial \hat{q}_t}{\partial t} = -\langle \nabla \log q_t, v\rangle - \langle \nabla, v\rangle.
\end{align}
\begin{proposition}
    For the evolution of the density $q_t = \exp(-E_t)/Z_t$, the solution of equation \eqref{eq:projection} is
    \begin{align}
        v^* = -\nabla \frac{\partial E_t}{\partial t}.
    \label{eq:main_eq}
    \end{align}
    \label{prop:main}
\end{proposition}
(See proof in Appendix \ref{app:proof_of_main}). 
This formula is the main development of our work and in the next section we discuss its practical implications.
In a similar way, we can project the evolution of the density
\begin{align}
    v^\star = \maxargmax_{v\in L^2_q: \norm{v}=1} \big\langle \dot{q}, -\langle\nabla, qv\rangle \big\rangle_{L^2} = \maxargmax_{v\in L^2_q: \norm{v}=1} \big\langle \nabla\dot{q}, v \big\rangle_{L^2_q} = \nabla\dot{q},
\end{align}
which is related to the gradient flows in the Wasserstein Riemannian manifold \citep{otto2001geometry, benamou2000computational}.
These two vector fields are equivalent when the distribution follows the gradient of some functional $F$.
\begin{proposition}
    Consider the functional $F = \int f(q)$, which we can optimize either w.r.t. $q = \exp(-E)/Z$ or w.r.t. $E$. When the evolution of the density (energy) is defined by the Frechet derivatve of $F$, we have $v^*=v^\star$.
    \label{prop:parameterizaiton}
\end{proposition}
(See proof in Appendix \ref{app:proof_of_parametrization}).
This preposition gives us a reasonable result. Namely, the dynamics of the particles is independent of the distribution parameterization when the parameterization is dense in the corresponding spaces.

Another motivation for the derived formula is that it can be approximated by the Persistent Contrastive Divergence with the Langevin dynamics.
\begin{proposition}
    The updates of the particles following $v^* = -\nabla \frac{\partial E}{\partial t}$ can be approximated as
    \begin{align}
        x_{t+dt} = x_t - \nabla_{x_t} E_{t+dt}(x_t) + \sqrt{2}\eps, \;\;\; \eps \sim \Normal(0,1).
    \end{align}
    \label{prop:pcd}
\end{proposition}
\vspace{-10pt}
(See derivations in Appendix \ref{app:proof_of_pcd}).
In the following section, we will see that the derived formula $v^* = - \partial E/\partial t$ allows for different approximations, which avoid any stochasticity in the updates.

\section{Numeric approximations of the particle dynamics}
\label{sec:approximations}

First, we consider the approximations that follow straightforwardly by discretizing formula \eqref{eq:main_eq}.
Discretizing the energy update, we have
\begin{align}
    v^*(x) = -\nabla_x \frac{\partial E_t(x)}{\partial t} \approx \frac{1}{dt}\bigg[-\nabla(E_{t+dt}(x)- E_t(x))\bigg] = v_\alpha(x).
\end{align}
The update of each individual particle can be written as
\begin{align}
    x_{t+dt} \approx x_t + dt \cdot v_\alpha(x_t) = x_t -\nabla_{x_t} E(x_t,\theta(t+dt))+\nabla_{x_t} E_t(x_t,\theta(t)),
\end{align}

We denote this update rule \textbf{Method $\boldsymbol{\alpha}$}.
This method is basically the simulation of the Langevin dynamics, but in a deterministic way.
Indeed, taking $p(\cdot) \propto \exp(-E(\cdot, \theta(t+dt))$ in equation \eqref{eq:det_langevin}, we obtain the same formula up to the choice of the step-size.
Intuitively, this procedure tries to match the new log-density following its gradient and, at the same time, unmatch the old log-density.
We describe the full training procedure in Algorithm \ref{alg:alpha_beta}.

The second option, for the parametric models, is to discretize the update of parameters instead:
\begin{align}
    v^* = & -\nabla \frac{\partial E}{\partial t} = -\nabla\bigg\langle  \nabla_\theta E(\cdot,\theta), \frac{\partial \theta}{\partial t}\bigg\rangle \approx -\frac{1}{d t}\nabla\bigg\langle  \nabla_\theta E(\cdot,\theta(t)), \theta(t+dt) - \theta(t)\bigg\rangle = v_\beta.
\end{align}
This formula gives us another update rule, which we call \textbf{Method $\boldsymbol{\beta}$}:
\begin{align}
    x_{t+dt}\approx x_t + dt\cdot v_\beta(x_t) = x_t -\nabla_{x_t}\langle  \nabla_\theta E(x_t,\theta(t)),\theta(t+dt) - \theta(t)\rangle.
    \label{eq:beta_v}
\end{align}
Unlike method $\alpha$, this method is different from deterministic Langevin, and we will return to its intuition in a bit.
For the full training procedure, see Algorithm \ref{alg:alpha_beta}.

\begin{algorithm}[H]
  \caption{Methods $\alpha, \beta$}
  \begin{algorithmic}  
    \REQUIRE{samples from the target distribution $p(x)$}
    \STATE get initial samples $\{x^{(i)}_0\}_{i=1}^n \sim q_{\theta(0)}(x) \propto \exp(-E(x,\theta(0)))$
    \FOR{$t \in [0,\ldots,T]$}
        \STATE estimate $-\nabla_\theta \KL(p,q_\theta) = -\nabla_\theta\big[ \mean_{x\sim p} E(x,\theta) - \mean_{x\sim q_\theta} E(x,\theta)\big]$
        \STATE update parameters $\theta(t+dt) = \text{Optimizer}\big[\theta(t), -\nabla_\theta \KL(p,q_\theta)\big]$
        \STATE update samples $x^{(i)}_{t+dt} = x^{(i)}_{t} + dt \cdot v_{\alpha,\beta}(x^{(i)}_{t})$ (see the formulas for $v_\alpha$ and $v_\beta$ in the text)
    \ENDFOR
    \RETURN{trained density model $q_{\theta(T)}(x) \propto \exp(-E(x,\theta(T)))$, final set of samples $\{x^{(i)}_T\}_{i=1}^n$}
  \end{algorithmic}
  \label{alg:alpha_beta}
\end{algorithm}
\vspace{-10pt}

The third option we consider is the non-parametric updates, which we derive approximating the energy gradient in RKHS $\mathcal{H}$ with kernel $k$.
To minimize the KL-divergence we first take the Frechet derivative w.r.t. the energy $E$ along some direction $h$ and use the fact that $\mathcal{H}$ is actually dense in $L^2_q$ \citep{duncan2019geometry}. Then, using the reproducing property of $k$, we can formulate the directional derivative as an action of a linear operator:
\begin{align}
    \mathrm{diff}\KL(p,q)[h] & = \langle p/q - 1, h \rangle_{L^2_q} = \langle \mean_{x\sim p} k(x,\cdot) - \mean_{x\sim q} k(x,\cdot), h\rangle_{\mathcal{H}}
    \label{eq:RKHS_diff}
\end{align}
Following \citep{gretton2012kernel}, we see that $\mu_p = \mean_{x\sim p} k(x,\cdot) \in \mathcal{H}$ if $\mean_{x\sim p}\sqrt{k(x,x)} <\infty$. Hence, we can choose the direction $h \in \mathcal{H}$ matching the gradient.

The gradient then defines the vector field, which we denote as \textbf{Method $\boldsymbol{\gamma}$}:
\begin{align}
    v^* = - \nabla \frac{\partial E}{\partial t} \approx \underbrace{\mean_{x\sim p} \nabla k(x,\cdot)}_{\text{attraction to data}} - \underbrace{\mean_{x\sim q_t} \nabla k(x,\cdot)}_{\text{repulsion between particles}} = v_\gamma.
    \label{eq:gamma_v}
\end{align}
Intuitively, the particles are attracted to the dataset and repelled from each other.
Also, this vector field coincides with the MMD gradient flow  \citep{arbel2019maximum}, which is derived from a different perspective.

\begin{proposition}
    The convergence of the dynamics \eqref{eq:gamma_v} is described as:
    \begin{align}
        \frac{d}{dt} \KL(p,q_t) = -\mathrm{MMD}_k(p,q_t)^2.
    \end{align}
    \label{prop:mmd}
    \vspace{-10pt}
\end{proposition}
(See proof in Appendix \ref{app:proof_of_mmd}).
Hence, the KL-divergence between the target and the current approximation reduces proportionally to the squared MMD between these distribution.
The process stops when $\text{MMD}_k(p,q_t) = 0$. 
If the kernel is expressive enough (is universal), then $q$ converges to $p$.

We now return to the intuition of method $\beta$.
Taking the formula \eqref{eq:beta_v} and assuming that the updates of the parameters follows the gradient descent, we have
\begin{align}
    v_\beta = & -\frac{1}{dt}\nabla\bigg\langle  \nabla_\theta E(\cdot,\theta), d\theta\bigg\rangle = \nabla \bigg\langle \nabla_\theta E(\cdot,\theta), \mean_{x\sim p} \nabla_\theta E(x,\theta) - \mean_{x\sim q_t} \nabla_\theta E(x,\theta) \bigg\rangle = \\
    = & \nabla \bigg[\mean_{x\sim p} \langle\nabla_\theta E(\cdot,\theta), \nabla_\theta E(x,\theta) \rangle - \mean_{x\sim q_t} \langle\nabla_\theta E(\cdot,\theta), \nabla_\theta E(x,\theta) \rangle\bigg] = v_\gamma.
\end{align}
Thus, we see, that method $\beta$ is essentially method $\gamma$, but with the Neural Tangent Kernel $k_\theta(x,y) = \langle\nabla_\theta E(x,\theta), \nabla_\theta E(y,\theta) \rangle$.
Hence, it operates by targeting the distribution of data rather than approximating the current energy model.

This connection has two potential benefits for method $\gamma$, which has the well-known downsides of the kernel methods.
The first one is the scaling to higher dimensions since NTK could be more expressive than conventional kernels like RBF.
The second benefit is the scaling in terms of batch size since the scalar product kernel allows for efficient parallel computation of the vector field. 
We describe the full procedure in Algorithm \ref{alg:gamma}.

\begin{algorithm}[H]
  \caption{Method $\gamma$}
  \begin{algorithmic}  
    \REQUIRE{samples from the target distribution $p(x)$}
    \STATE get initial samples $\{x^{(i)}_0\}_{i=1}^n \sim q_{\theta(0)}(x) \propto \exp(-E(x,\theta(0)))$
    \FOR{$t \in [0,\ldots,T]$}
        \STATE estimate $\nabla_\theta \KL(p,q_t) = \nabla_\theta\big[ \mean_{x\sim p} E(x,\theta) - \mean_{x\sim q_t} E(x,\theta)\big]$
        \STATE update samples $x^{(i)}_{t+dt} = x^{(i)}_{t} + dt \cdot \nabla_{x^{(i)}_{t}} \bigg\langle \nabla_\theta E(x^{(i)}_{t},\theta), \nabla_\theta \KL(p,q_t) \bigg\rangle$
    \ENDFOR
    \RETURN{final set of samples $\{x^{(i)}_T\}_{i=1}^n$}
  \end{algorithmic}
  \label{alg:gamma}
\end{algorithm}
\vspace{-10pt}

\section[Empirical evaluation]{Empirical evaluation \footnote[1]{The code reproducing experiments is available at \href{https://github.com/necludov/particle-EBMs}{github.com/necludov/particle-EBMs}}}

\begin{figure}
    \centering
    \includegraphics[width=0.99\linewidth]{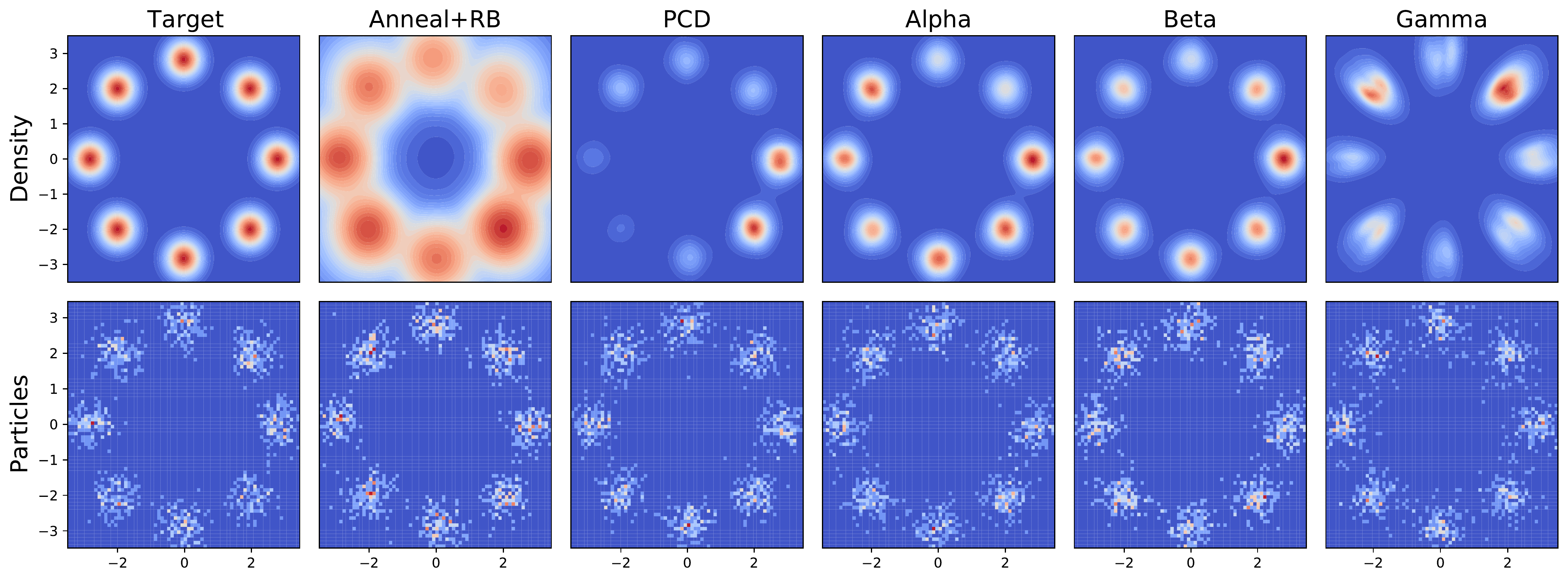}
    \caption{The top row depicts the learned densities for different approaches. Since method $\gamma$ doesn't yield the parametric model for the energy, we integrate the energy numerically using $\partial E/\partial t$ from \eqref{eq:RKHS_diff}. The bottom row depicts the histograms of samples obtained in the end of learning procedures.}
    \label{fig:densities}
\vspace{-10pt}
\end{figure}

We empirically test the proposed methods $\alpha, \beta, \gamma$ and compare them against conventional approaches: PCD \citep{tieleman2009using} and sampling with Replay Buffer \citep{du2019implicit}.
We found that for the stability of Replay Buffer it is important to reduce the noise magnitude (as also proposed in \citep{du2019implicit}).
That, however, yields sampling from an annealed target.
For both ``PCD'' and ``Anneal + RB'', we make $20$ steps of stochastic Langevin on every iteration.
For our methods we propagate the particles along the corresponding vector fields and make additional $10$ steps of stochastic Langevin to alleviate possible numerical errors.
For the target distribution we take toy 2-d distribution, and try to match it with 2-layer fully-connected neural network ($300$ hidden units, Swish activations \citep{ramachandran2017searching}).
For method gamma, we found that using the same parameters $\theta$ throughout the learning leads to degenerate solutions.
Therefore, we sample using the kernel $k_\theta(x,y) = \mean_{\theta \sim \pi_0}\langle\nabla_\theta E(x,\theta), \nabla_\theta E(y,\theta) \rangle$, where parameters $\theta$ are sampled from the initialization distribution $\pi_0$ at each iteration, i.e. we use unlearned random networks to propagate particles.

\begin{wrapfigure}{r}{0.5\textwidth}
\vspace{-15pt}
  \begin{center}
    \includegraphics[width=0.5\textwidth]{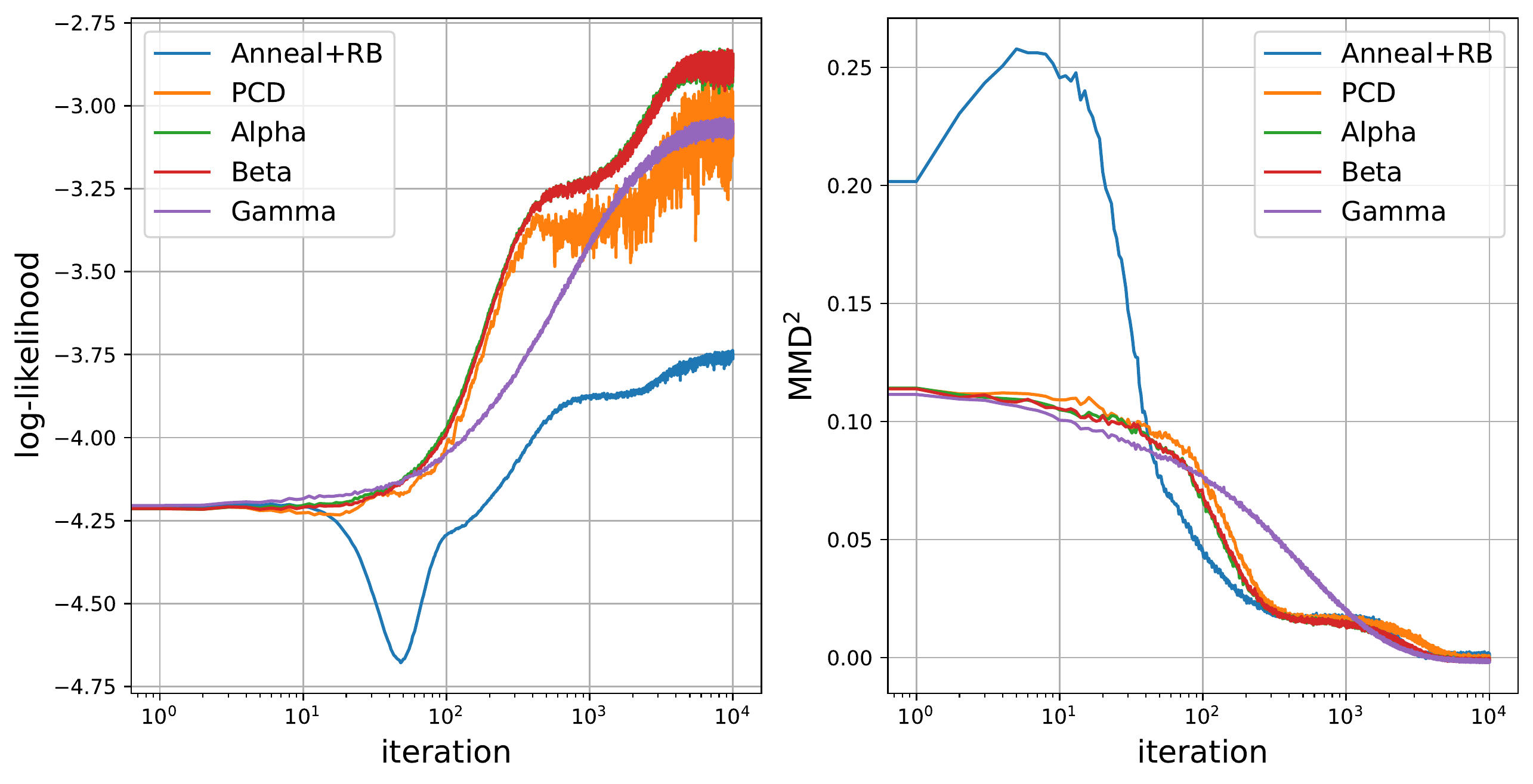}
  \end{center}
  \caption{Performance of the models throughout the training. The quality of the energy as measured by the log-likelihood, and the MMD$^2$ between current set of particles and the training batch. Note that the Alpha's plot is just under the Beta's plot.}
  \label{fig:training}
\vspace{-0pt}
\end{wrapfigure}
In Fig. \ref{fig:densities}, we demonstrate the learned densities and the particles.
In Fig. \ref{fig:training}, we report the metrics for models and samples averaging over $10$ independent runs. We don't report standard errors to keep the plots readable.
As we see, only $\alpha$ and $\beta$ nicely capture all of the modes. We found the learning via PCD to be the most unstable and unable to capture all of the modes in most cases. 
Annealing with the Replay Buffer is the most stable method in our experiments. 
However, it yields the density with scaled temperature due to the incorrect noise magnitude.
Finally, we conclude that the proposed methods demonstrate better performance with a lower computational budget.
Interestingly, all of the methods manage to match the final set of particles with the target distribution regardless of the learned energy.

\section{Conclusion}
We approach the problem of sampling from a distribution evolving in time, which is especially important in the context of energy-based learning. 
Our main contribution is the approximate formula for the vector field that propagates the particles matching the evolution of the distribution.
We demonstrate that this formula yields several reasonable algorithms connected to deterministic Langevin and MMD gradient flows.
Intuitively, method $\alpha$ moves the particles matching the new energy and, at the same time, unmatching the old energy.
In contrast, methods $\beta$ and $\gamma$ propel particles aiming the target data distribution and repelling particles from each other to cover the state-space.
Finally, we show that our deterministic approach can be favorable in practice for learning energy-based models.

\bibliography{iclr2022_conference}
\bibliographystyle{iclr2022_conference}

\newpage
\appendix

\section{Proof of proposition \ref{prop:main}}
\label{app:proof_of_main}

\begin{proposition*}
    The solution of 
    \begin{align}
        v^* = \maxargmax_{v\in L^2_q: \norm{v}=1} \bigg\langle \frac{\partial}{\partial t} \log q_t, \frac{\partial}{\partial t} \log \hat{q}_t \bigg\rangle_{L_q^2}
        \label{eq_app:projection}
    \end{align}
    is
    $v^* = -\nabla \frac{\partial E}{\partial t}$.
\end{proposition*}
\begin{proof}
    The evolution of $\hat{q}_t$ is evolution defined by the vector field $v$, i.e.
    \begin{align}
        \frac{\partial}{\partial t}\log \hat{q}_t = \frac{1}{\hat{q}_t}\frac{\partial \hat{q}_t}{\partial t} = -\langle \nabla \log q_t, v\rangle - \langle \nabla, v\rangle,
    \end{align}
    and the first argument of the scalar product is defined by the updates of the energy
    \begin{align}
        \frac{\partial}{\partial t}\log q_t = -\frac{\partial E}{\partial t} + \mean_{q_t} \frac{\partial E}{\partial t}.
    \end{align}
    We rewrite the scalar product in \eqref{eq_app:projection} as
    \begin{align}
        \bigg\langle \frac{\partial}{\partial t} \log q_t, \frac{\partial}{\partial t} \log \hat{q}_t \bigg\rangle_{L_q^2} & = \int dx q \bigg[\frac{\partial E}{\partial t} - \mean_{q_t} \frac{\partial E}{\partial t}\bigg]\bigg[\langle \nabla \log q_t, v\rangle + \langle \nabla, v\rangle\bigg] = \\ 
        & =
        \int dx q \frac{\partial E}{\partial t}\bigg[\langle \nabla \log q_t, v\rangle + \langle \nabla, v\rangle\bigg] - \mean_{q_t} \frac{\partial E}{\partial t}\int dx q \frac{1}{q} \frac{\partial q}{\partial t} = \\
        & = \int dx q \frac{\partial E}{\partial t}\bigg[\langle \nabla \log q_t, v\rangle + \langle \nabla, v\rangle\bigg].
    \end{align}
    Integrating by parts, we have
    \begin{align}
        \bigg\langle \frac{\partial}{\partial t} \log q_t, \frac{\partial}{\partial t} \log \hat{q}_t \bigg\rangle_{L_q^2} & = \int dx \bigg\langle \frac{\partial E}{\partial t} \nabla q - \nabla \big(q\frac{\partial E}{\partial t}\big), v \bigg\rangle = \\ 
        &= \int dx q \bigg\langle -\nabla \frac{\partial E}{\partial t}, v \bigg\rangle = \bigg\langle -\nabla \frac{\partial E}{\partial t}, v \bigg\rangle_{L_q^2}.
    \end{align}
\end{proof}

\section{Proof of proposition \ref{prop:parameterizaiton}}
\label{app:proof_of_parametrization}

\begin{proposition*}
    Consider the functional $F = \int f(q)$, which we can optimize either w.r.t. $q = \exp(-E)/Z$ or w.r.t. $E$. Vector fields $v^*=v^\star$ when the evolution of the density (energy) is defined by the Frechet derivatve of $F$.
\end{proposition*}
\begin{proof}
    The directional derivative (along the direction $h \in L^2$) is
    \begin{align}
        \mathrm{diff}F(q)[h] = \bigg\langle \frac{\delta F(q)}{\delta q},h \bigg\rangle_{L^2}, \;\;\; \frac{\delta F(q)}{\delta q}h = \frac{d}{d\eps}f(q+\eps h)\big\vert_{\eps=0}.
    \end{align}
    We can think of $\frac{\delta F(q)}{\delta q}$ as of the formal symbolic application of differention rules.
    Then
    \begin{align}
        v^\star = \nabla \frac{\partial q}{\partial t} = \nabla \frac{\delta F(q)}{\delta q},
    \end{align}
    which coincides with the vector field given by the Otto Calculus \citep{otto2001geometry}.
    For the energy, we consider the direction $h \in L^2_q$, then we have
    \begin{align}
        \mathrm{diff}F(E)[h] = \int \frac{\delta F(q)}{\delta q}\frac{\delta q(E)}{\delta E}h = \bigg\langle -\frac{\delta F(q)}{\delta q} + \mean_q \frac{\delta F(q)}{\delta q},h \bigg\rangle_{L^2_q}.
    \end{align}
    Finally, we see that the two derivatives yield the same vector-field.
    \begin{align}
        v^* = -\nabla \frac{\partial E}{\partial t} = -\nabla\bigg[-\frac{\delta F(q)}{\delta q} + \mean_q \frac{\delta F(q)}{\delta q}\bigg] = v^\star
    \end{align}
\end{proof}

\section{Proof of proposition \ref{prop:pcd}}
\label{app:proof_of_pcd}

\begin{proposition*}
    The vector field $v^* = -\nabla \frac{\partial E}{\partial t}$ may be approximated by the ``conventional update rule'' of the particles following the Langevin dynamics targeting the updated density $q_{t+dt}$.
\end{proposition*}
\begin{proof}
    Let's approximate the vector field $v^*$ as 
    \begin{align}
        v^*\approx -\nabla \frac{1}{dt}\bigg[E_{t+dt}-E_t\bigg] = \frac{1}{dt}\bigg[\nabla\log q_{t+dt}-\nabla\log q_t\bigg],
    \end{align}
    Then the evolution of the density is described by the FP equation:
    \begin{align}
        \dot{q} = -\langle\nabla, q_t v^*\rangle = -\langle\nabla, q_t \frac{1}{dt}\nabla\log q_{t+dt}\rangle + \frac{1}{dt}\Delta q_t.
    \end{align}
    Hence the evolution of particles can be described by the Ito equation
    \begin{align}
        x_{t+dt'} = x_t + dt'\frac{1}{dt}\nabla\log q_{t+dt}(x_t) + \sqrt{\frac{2}{dt}}dW_t,
    \end{align}
    where $dW_t$ is the Wiener process, which can be simulated by the normal random variable $\Normal(0,dt')$.
    Taking $dt' = dt$, we have the conventional update rule (up to the step size choice)
    \begin{align}
        x_{t+dt} = x_t + \nabla\log q_{t+dt}(x_t) + \sqrt{2}\Normal(0,1) = x_t - \nabla E_{t+dt}(x_t) + \sqrt{2}\Normal(0,1).
    \end{align}
\end{proof}

\section{Proof of Proposition \ref{prop:mmd}}
\label{app:proof_of_mmd}

\begin{proposition*}
    The convergence of \eqref{eq:gamma_v} is described by the equation:
    \begin{align}
        \frac{d}{dt} \KL(p,q_t) = -\mathrm{MMD}_k(p,q_t)^2.
    \end{align}
\end{proposition*}
\begin{proof}
    \begin{align}
        \mathrm{diff}\KL(p,q)[h] & = \langle p/q - 1, h \rangle_{L^2_q} = \langle \underbrace{\mean_{x\sim p} k(x,\cdot) - \mean_{x\sim q} k(x,\cdot)}_{\partial E/\partial t}, h\rangle_{\mathcal{H}}
    \end{align}
    \begin{align}
        \frac{d}{dt} \KL(p,q_t) = & -\mean_{x\sim p} \frac{d}{dt} \log q_t(x) = \mean_{x\sim p} \frac{\partial}{\partial t} E(x) - \mean_{x\sim q_t} \frac{\partial}{\partial t} E(x) = \\
        = & \mean_{x'\sim p}\big[\mean_{x\sim q_t} k(x,x') - \mean_{x\sim p} k(x,x')\big] - \\
         & - \mean_{x'\sim q_t}\big[\mean_{x\sim q_t} k(x,x') - \mean_{x\sim p} k(x,x')\big] = \\
        = & -\mean_{x,x'\sim p} k(x,x') + 2\mean_{x\sim p, x'\sim q_t} k(x,x') -\mean_{x,x'\sim q_t} k(x,x') = \\
        = & -\mathrm{MMD}_k(p,q_t)^2.
    \end{align}
\end{proof}

\end{document}